\documentclass{article}

\usepackage{microtype}
\usepackage{graphicx}
\usepackage{subfigure}
\usepackage{booktabs} 

\usepackage{hyperref}

 \usepackage[accepted]{sysml2019}

\sysmltitlerunning{TOCO: Tolerance-based Compression Framework}

\usepackage{amsmath}
\usepackage{amsfonts}
\renewcommand{\baselinestretch}{1.0}

\newcommand{\blue}[1] {\textcolor{black}{#1}}
\newcommand{\red}[1] {\textcolor{black}{#1}}

\usepackage{tikz}
\newcommand*\circled[1]{\tikz[baseline=(char.base)]{
            \node[shape=circle,fill,inner sep=1pt] (char) {\textcolor{white}{#1}};}}

\usepackage[utf8]{inputenc}
\usepackage[english]{babel}
\usepackage{amsthm}
\newtheorem{theorem}{Theorem}

\newtheorem{lemma}[theorem]{Lemma}
\begin{document}

\onecolumn
\sysmltitle{TOCO: A Framework for Compressing Neural Network Models Based on Tolerance Analysis}




\begin{sysmlauthorlist}
\sysmlauthor{Soroosh Khoram}{aff1}
\sysmlauthor{Jing Li}{aff1}
\end{sysmlauthorlist}

\sysmlaffiliation{aff1}{Department of Electrical and Computer Engineering, University of Wisconsin-Madison, Madison, WI, USA.}

\sysmlcorrespondingauthor{Jing Li}{jli@ece.wisc.edu}

\sysmlkeywords{Machine Learning, SysML}

\vskip 0.3in

\begin{abstract}


\red{Neural network compression methods have enabled deploying large models on emerging edge devices with little cost, by adapting already-trained models to the constraints of these devices. The rapid development of AI-capable edge devices with limited computation and storage requires streamlined methodologies that can efficiently satisfy the constraints of different devices. In contrast, existing methods often rely on heuristic and manual adjustments to maintain accuracy, support only coarse compression policies, or target specific device constraints that limit their applicability. We address these limitations by proposing the TOlerance-based COmpression (TOCO) framework. TOCO uses an in-depth analysis of the model, to maintain the accuracy, in an active learning system. The results of the analysis are tolerances that can be used to perform compression in a fine-grained manner. Finally, by decoupling compression from the tolerance analysis, TOCO allows flexibility to changes in the hardware.}



\end{abstract}



\printAffiliationsAndNotice{}  

\section{Introduction}
\red{The success of Deep Learning has lead to rapid development of numerous hardware platforms to deploy them in edge computing scenarios \cite{yu2017scalpel,parashar2017scnn}. However, modern Deep Neural Networks (DNNs) often store hundreds of millions of parameters and perform billions of computations for inference. Thus, due to the limited computation and storage available in edge computing systems, models need to be summarized. Neural network compression methods are uniquely suited for this task as they can adapt already-trained models to the constraints of different target hardware without the significant cost of training.}


\red{Previous works have successfully compressed neural network models with strict computation and storage constraints. This is done by optimizing the model with constraints of the hardware by iteratively pruning \cite{han2015deep} or quantizing \cite{zhou2017incremental} the parameters and tuning them. They further usually initialize the optimization with a pre-trained model, to avoid cost of training the model from scratch, and use its parameter values and derivatives to decide which parameters to prune or quantize. While effective, existing methods can have important limitations. They use heuristics that might be inaccurate and affect the accuracy of the compressed model. Further, they apply  encoding policies for compression in coarse granularities which has been shown inefficient \cite{park2018energy}. Finally, they can be tightly coupled to the constraints of a specific hardware, preventing from easy application to a wide range of hardware features and constraints.}


\red{In this work, we propose Tolerance-based Compression (TOCO) which streamlines DNN model compression for deployment on different resource constrained hardware while addressing limitations of existing methods. Figure \ref{fig:overview} shows an overview of TOCO. It uses a committee \circled{1} comprising the uncompressed, pretrained model and the compressed model to identify disagreements between the two over the training set using a Query-By-Committee (QBC) component \circled{3}. Among these, QBC chooses representative ones based on similarity analysis \circled{2}, as they are potentially the most informative. Finally, the selected samples are used in tolerance analysis \circled{4} which produces perturbation bounds on the network parameters. Tolerances are then combined with the hardware constraints to generate the compressed model. Tolerance analysis replaces conventional heuristics and maintains accuracy during compression and allows application of encoding policies in a fine-grained manner. Furthermore, it decouples the compression method from the hardware constraints, allowing adaptation to a wide range of target hardware. We will further elaborate on these components in Section 3.1.}


\begin{figure}[tb]
    \centering
    \includegraphics[width=0.85\linewidth]{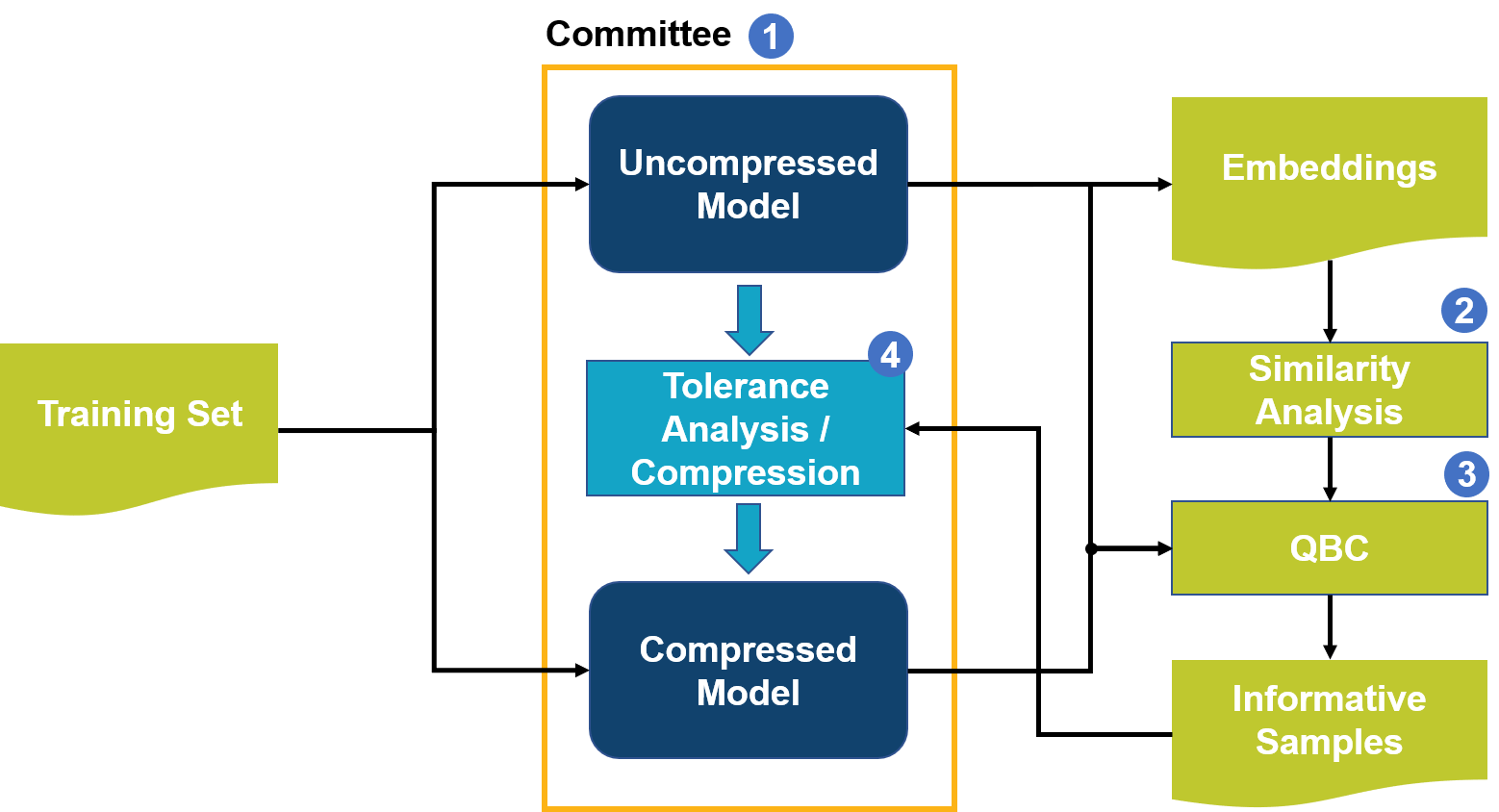}
    \caption{Overview of the main components of TOCO}
    \label{fig:overview}
\end{figure}

We will demonstrate the general applicability of our framework using different compression methods as well as a variety of datasets including CIFAR \cite{cifar} and ImageNet \cite{ILSVRC15}, and neural network models such as ResNet50 \cite{resnet} and VGG16 \cite{simonyan2014very}. In these experiments, we compare the results produced through our proposed framework with several existing compression techniques. We will further compare the speedups achieved using this approach when deploying onto different hardware platforms with an existing heuristic approach under the same accuracy. We will show that the proposed framework outperforms previous heuristic methods through a quick exploration of the solution space, and provides smaller models with faster inference speeds.

\section{Related Works}
Various methods have been proposed in the literature for reducing storage and computation costs of DNNs. Prominent such methods include pruning, quantization, and meta-learning. Pruning eliminates unnecessary network parameters reducing both storage and the number of computations necessary for inference \cite{li2016pruning,molchanov2016pruning}. Quantization uses imprecise, but short encodings to represent parameters \cite{he2016deep,courbariaux2015binaryconnect,hubara2016quantized}. In addition to reducing storage, quantization can reduce the necessary precision of computations, allowing inference on more primitive devices. Meta-learning can be used to learn hyper-parameters of a network, such as the number of filters, to reduce the overall number of parameters and computations \cite{he2018adc,smithson2016neural,he2018amc}. 

Pruning chooses which parameters to prune using on a measure of their importance to accuracy. This measure is often based on heuristics such as their absolute value \cite{he2016deep,yu2017scalpel}. Such heuristics can be inaccurate. For example, removing some small parameters can affect accuracy considerably \cite{molchanov2016pruning}. 

Quantization samples parameters from a small, discrete subset of $\mathbb{R}$. These parameters can then be encoded using only a few bits. However, many quantization methods apply encoding policies in large granularities \cite{hubara2016quantized,zhou2016dorefa}. For example, they encode all parameters in a layer or an entire network using the same number of bits and with the same precision. Such policies have been shown inefficient \cite{park2018energy,khoram2018adaptive} as only a small fraction of parameters usually need to be represented in high precision \cite{park2018energy}.

Meta-learning can be used to balance inference efficiency and accuracy. Such search methods have been shown productive in finding optimal neural network architectures and training hyper-parameters \cite{stanley2002evolving,meyerevolving}. Deep Reinforcement Learning (DRL) has been applied to choose smaller filters \cite{he2018adc} or pruning rates \cite{he2018amc}. These methods implicitly learn the importance of parameters, and less frequently rely on heuristics. But, they also apply encoding policies in large granularities and often do not consider target hardware, assuming GPU by default. This can limit their flexibility in adapting to novel hardware. Further, they can be costly as each actor-critic data point may require complete training and testing of the network, and they still might not be applicable to other methods due to the limitations of DRL.

The proposed framework, called TOCO, automates the process of searching for optimal compressions of a DNN model for a target hardware and addresses the limitations of previous works. Unlike previous automated methods that are mainly applied to pruning targeted for GPU \cite{he2018amc}, TOCO subsumes a wide range of compression techniques as well as hardware. Further, it replaces common heuristics with a quantitative analysis of perturbation bounds for network parameters to decide how aggressively to apply compression without degrading the accuracy. Previously, \citet{lecun1990optimal} proposed a method for computing similar bounds based on the loss Hessian which is still used in some practical applications \cite{theis2018faster}. However, their technique is expensive since it requires computation of second derivatives, is applicable only to pruning, and also disregards the target hardware. By learning individual perturbation bounds of parameters, TOCO is also able to introduce encoding policies in fine granularities that is a feature absent from previous works. 

\section{Methodology of TOCO}
In this section, we further elaborate on the main components of TOCO from Figure \ref{fig:overview}. We fill first present an overview of TOCO and discuss the role of sample selection in the TOCO design. Then, we present the tolerance analysis / compression component in detail.

\subsection{Overview}
\blue{TOCO computes tolerances for parameters of the network, which are perturbation bounds under a loss constraint, and uses them to compress the model. The tolerance analysis in TOCO (discussed in Section \ref{sec:tolerance}) relies on gradient computation. Accurately computing gradients requires backpropagation for all of the training dataset which can be very expensive. Moreover, assuming the model has been trained or tuned, backpropagation results for many dataset elements may be small and do not meaningfully contribute to the overall gradient. As such, it is advantageous to only select a batch of the informative samples from the dataset for gradient computations. This opens an opportunity to select samples carefully to improve the confidence of the tolerances. To obtain more reliable tolerances, we may choose samples that are closer to the boundary of being misclassified. That is because samples that can be reliably classified correctly are hopefully resistant to perturbations in the model. To select uncertain samples, we can take inspirations from active learning methods.} 

\blue{Active learning explores learning settings where large but unlabeled datasets exists and data sample labels can be obtained only by  querying an oracle, at a cost. The learning algorithm has to query labels from the oracle and update the model accordingly. Naturally, many samples (e.g. samples similar to ones already observed or outliers) do not contribute significant information to the learning task and do not warrant the cost of querying their labels. To minimize the cost of the learning algorithm, the objective of active learning methods is then to identify samples that contribute the most information. }

\blue{From this description, it is easy to incorporate an active learning design in TOCO. In this context, the unknown labels are the gradients and the oracle is backpropagation. Various active learning methods exist to accomplish the task of selecting informative samples given a (partially-trained) model. Most such algorithms rely on computing some measure of uncertainty for the classification of the samples. High uncertainty in a sample means that the model cannot be unambiguously placed in a class by the model. Here, we choose the method called Query By Committee to measure and rank the uncertainty of data samples.}

\blue{Query By Committee (QBC) assumes a committee of models trained on the same classification task, and uncertainty is defined as the degree of disagreement between these models in classifying a data sample. Disagreement can be measured using the Kullback-Leibler (KL) divergence. For a data sample, each member of the committee outputs a posterior distribution for the class assignments. The KL divergence can be used to measure the ``distance'' between these posteriors. Previous works \cite{mccallumzy1998employing} have further used similarity-weighted QBC to identify samples that are not only uncertain, but also representative. Uncertainty measures, when used in isolation can be distracted by outliers. To avoid this issue, we can combine the uncertainty measure (i.e. the KL divergence) with the similarity of data samples to other dataset samples. One measure of similarity used in previous works \cite{mccallumzy1998employing} is $\exp(-D(x, c(x)))$, where x is the dataset element, $c(x)$ is the cluster center for the class to which x belongs, and $D(x,c(x))$ is the distance between $x$ and $c(x)$.}

\blue{Similarity-weighted QBC is used in TOCO to select samples for tolerance computation as Figure \ref{fig:overview} shows. The committee \circled{1} comprises the original uncompressed model and the compressed version of the same model produced by TOCO. Disagreements between these two models are a result of the compression. As such, using uncertain samples identified by this committee can therefore help minimize the effect of compression on accuracy. Before measuring the uncertainties however, we analyze the similarities \circled{2} between data samples to select representative ones. The similarity between samples are measured using the embeddings of the samples extracted from the last layer of the uncompressed model. These similarities can be precomputed before the compression once and reused at each step of the compression when ranking the dataset samples. The similarity analysis results and the KL divergence of the committee are next used to rank the training set by the similarity-weighted QBC \circled{3} which chooses the informative data samples. Finally, these informative samples are used by TOCO \circled{4} to compress the model further.}

\blue{This design has the further advantage of providing more informative gradients when the compressed model is near the original model and gradients are small. While the gradients can be small near an optima obtained during training through backpropagation, as previous works in gradient-based compression have shown \cite{liu2019channel} they can be expressive in identifying unimportant weights. Furthermore, after a few rounds of compression which perturb the model away from the optima, this issue disappears, given we do not perform tuning. The problem may arise again if we are computing gradients on batches of the dataset (\citet{liu2019channel} used all of dataset for gradient computation) and performing tuning. In such a scenario, the tuning resets gradients back to near zero and random batch selection may result in random gradients that are unreliable for importance evaluation. This can be addressed through judicious choice of the batch through the proposed method.}

\blue{Finally, it should be noted that at the beginning of compression where the compressed and uncompressed models are similar, the KL divergence between the models are small. However, as we will discuss in the next section, tolerance analysis in TOCO is an iterative process. Thus, after the first iteration of the algorithm the divergence increases and can be used to identify disagreements.}

\subsection{Tolerance Analysis and Compression}
\label{sec:tolerance}
The key building blocks of TOCO tolerance analysis and compression are depicted in Figure \ref{fig:tolerance}, comprising three stages i.e. compression modeling, hardware-independent tolerance analysis, and hardware-dependent compression. The first stage formulates the complexity of the DNN model as a cost function which compression methods often aim to minimize. This cost informs the necessary perturbations of parameters needed for compression. Any perturbation of the parameters however needs to be made in a fashion that minimizes accuracy loss. To achieve this, in the second stage, TOCO learns bounds for these perturbation that allow for maximal reduction of the complexity cost while minimizing the accuracy loss. These bounds, which we refer to as tolerances, represent the sensitivity of each parameters to perturbation. Low tolerance, i.e. less sensitive, parameters may require high-precision encodings and vice versa.  Before encoding parameters though, in the final stage TOCO makes sure that they are compatible with the capabilities of the hardware to enable efficient storage and computations. As this is the only part of the analysis that depends on the hardware, TOCO decouples compression from the hardware and is easily portable across various hardware platforms. In the rest of this section, we first further explore the relationship between a compression model and the hardware to better illustrate the reasoning behind the composition of TOCO. Then, we present the three components of TOCO in more details.

\begin{figure}[tb]
    \centering
    \includegraphics[width=0.5\textwidth]{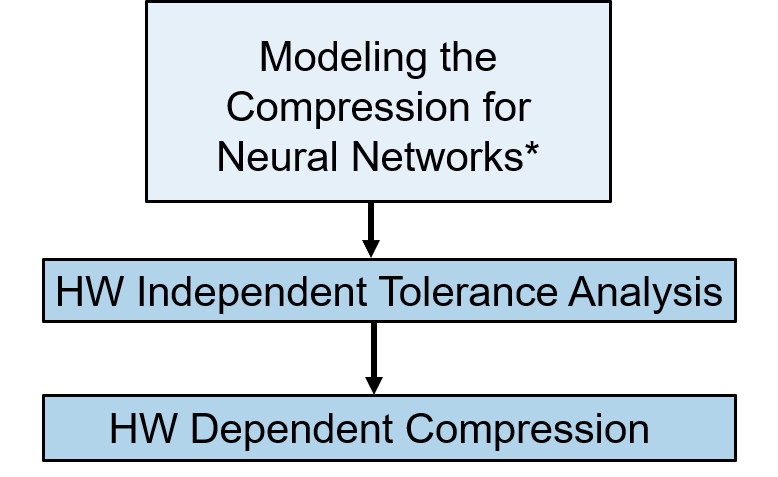}
    \caption{The key components of tolerance analysis in TOCO}
    \label{fig:tolerance}
\end{figure}

\subsubsection{Compression Methods vs. Deployed Hardware}
\label{sec:comp_vs_hw}
Compression encodes parameters of a network in order to reduce its hardware requirements. This is often done by defining and minimizing a measure of the complexity of the network like $\Phi(W, \mathcal{C})$. Here, $W$ is the vector of all network parameters and $\mathcal{C}$ is a finite set of values used for encodings. In this optimization, parameters $W$ are updated to cluster around elements of $\mathcal{C}$ while elements of $\mathcal{C}$ can be optimized to maintain accuracy. Finally, each parameter $\omega_i$ is set to its nearest element from $\mathcal{C}$ which corresponds to a unique binary symbol. These symbols are ultimately stored and are referenced during computations of inference.

The hardware characteristics of the compressed network (i.e. storage size, inference time, etc.) are determined by the number of unique symbols and their distribution across the network parameters. In fact, it is critical to hardware efficiency that their distribution follow certain memory access and parallelism rules. However, most existing works optimize only for the number of symbols. Distributions are either ignored or decided before compression. Conversely, TOCO separates minimizing $\Phi$ from the final parameter updates, allowing these distributions being considered.  TOCO learns the bounds to which each parameter can be updated for the compression method. But actual updates are made in conjunction with hardware characteristics. This allows TOCO to take hardware efficiency into account in a way unlike the standalone compression methods \cite{he2016deep}.

\subsubsection{Compression Model}

DNN compression methods minimize the complexity of the model which here we refer to as the compression model. This complexity can be formulated based on various definitions with different applications. Some of the prominent definitions include: overall bit count \cite{khoram2018adaptive}, $L_2$ distance \cite{han2015deep,choi2016towards}, and description length \cite{ullrich2017soft}. Our goal in designing TOCO is to be able to minimize the complexity regardless of how it is defined. This enables flexibly choosing the compression method. In TOCO, this is achieved by defining a generic compression model as below:
\begin{equation}
\label{eq:compmodel}
    \Phi(\mathbf{W}) = \sum_i \phi(\omega_i)
\end{equation}
Here, $\mathbf{W}=[\omega_1, …, \omega_n]^T$ is the vector of all network parameters, and $\Phi(\mathbf{W})$ is the compression model which represents the complexity of the network and is decomposed into the sum of $n$ smooth, convex, and lower bounded functions $\phi(\omega_i)$. This definition is compatible with all aforementioned definitions of complexity.  Finally, without loss of generality, we assume that $\forall i: -\overline{\omega} \leq \omega_i \leq \overline{\omega}$.

\subsubsection{Hardware-Independent Tolerance Analysis}
DNNs assimilate a certain level of redundancy which compression methods can use for reducing complexities. This means perturbations of these parameters after training may only have a small impact on the accuracy. Still perturbations in a direction that reduces complexity are often at odds with maintaining high accuracy. It is imperative that parameters are updated while minimizing the accuracy loss. As we discussed in previous sections however computing these optimal updates does not guarantee efficient implementation on the hardware. Therefore, TOCO computes perturbation bounds instead, independent of the hardware, for the parameters which can be used to compute the updates later while taking the hardware into account later. In this section, we first present computation of these perturbation bounds, which we refer to as tolerances, as an optimization problem and discuss how this problem can be solved.

\textbf{Problem Definition: }
TOCO finds perturbation bounds of the parameters such that the complexity can be minimized while maintaining the accuracy above a threshold. Here, we use the training loss as the proxy for accuracy. In other words, the goal here is to find the tolerance vector $\mathbf{T} = [\tau_1, ..., \tau_n]^T$ such that if the parameters $\mathbf{W}^0$ where to be perturbed by $\mathbf{T}$ the training loss would not drop below $\overline{\ell}$. Furthermore, we need these tolerances to allow the complexity measure $\Phi$ to be minimized. For the latter, we derive a lower bound on the complexity as a function of $\mathbf{T}$ and minimize this bound. Formally thus we can define the minimization problem to compute the tolerances as below:

\begin{gather}
    \label{eq:main_optimization}
    \min_{\mathbf{T}} \Phi(\mathbf{W}^0+\mathbf{T}) \\
    \label{eq:constraint}
    \forall \mathbf{W}~s.t.~|\omega_i-\omega_i^0| \leq |\tau_i|: \mathcal{L}(\mathbf{W}) \leq \overline{\ell}
\end{gather}

Here, $\mathcal{L}(\mathbf{W})$ indicates the training loss with the parameters $\mathbf{W}$. The variable $\mathcal{C}$ has been omitted for brevity as it is not modified here. In section 4, we will discuss how the loss bound, $\overline{\ell}$ is computed. Furthermore, we have assumed $\tau_i$ to be in the descent direction of $\phi(\omega_i^0)$.  Without loss of generality, we assume tolerances are positive and account for the descent directions separately. 
\begin{gather}
    \min_{\mathbf{T}} \Phi(\mathbf{W}^0+\mathbf{A}\circ\mathbf{T})=\sum_i\phi(\omega_i^0+\alpha_i\tau_i)\\
    \mathbf{A}=[\alpha_1, ...\alpha_n]^T\in\{\pm1\}^n
\end{gather}
Finally, we assume that if there exists a minima in the descent direction of $\phi(\omega_i^0+\alpha_i\tau_i)$, say $\omega_i^*$, then $\tau_i\leq|\omega_i^*-\omega_i^0|$. We note that by this definition, $\Phi(\mathbf{W}^0+\mathbf{A}\circ\mathbf{T})$ is a lower bound for $\Phi(\mathbf{W}^0+\mathbf{A}\circ\mathbf{\Delta W})$ when $0\leq \Delta \omega_i\leq \tau_i$. Consequently, the constraint of equation \ref{eq:constraint} can be written as:
\begin{equation}
    \label{eq:constraint2}
    \mathcal{L}(\mathbf{W}^0+\mathbf{A}\circ\mathbf{\Delta W}) \leq \overline{\ell}
\end{equation}
\textbf{Solution: }As a first step to solve this problem we need to simplify the condition of equation \ref{eq:constraint2}. First the loss function $\mathcal{L}$ is normally too complex and can impose expensive computations to solve this problem. Second this constraint is not directly defined over $\mathbf{T}$. We then use the Taylor expansion as a local estimation of the loss to address its complexity. We will discuss how the accuracy of this estimation is controlled in section \ref{sec:application}. We address the second issue by finding a bound over the loss as a function of $\mathbf{T}$. TOCO uses use the first component of Taylor expansion to estimate $\mathcal{L}$.
\begin{gather}
    \mathcal{L}(\mathbf{W}^0+\mathbf{A}\circ\mathbf{\Delta W})-\mathcal{L}(\mathbf{W}^0) \approx\\ \nabla_{\mathbf{W}}\mathcal{L}(\mathbf{W}^0)^T\mathbf{\Delta W} \leq \mathbf{G}^T \mathbf{T}
\end{gather}
Here, $\mathbf{G}=[g_1, ..., g_n]^T$ where $g_i=|[\nabla_{\mathbf{W}}\mathcal{L}(\mathbf{W}^0)]_i|$. Thus, we can guarantee the constraint (assuming that the linear estimation is sufficiently accurate) if:
\begin{equation}
    \mathbf{G}^T \mathbf{T} \leq \overline{\ell} - \mathcal{L}(\mathbf{W}^0) = \Delta\overline{\ell}
\end{equation}
This simplifies the original optimization problem as below:
\begin{gather}
    \min_{\mathbf{T}}\hat{\Phi}(\mathbf{T})=\Phi(\mathbf{W}^0+\mathbf{A}\circ\mathbf{T}) \\
    \mathbf{G}^T \mathbf{T} \leq \Delta\overline{\ell}
\end{gather}
We can solve this by writing its KKT conditions.
\begin{gather}
    \nabla_{\mathbf{T}}\hat{\Phi}+\lambda G = 0\\
    \lambda(\mathbf{G}^TT-\Delta\overline{\ell}) = 0\\
    \lambda \geq 0
\end{gather}
We can solve this system of equations to find $\lambda$ and $\mathbf{T}$. While without additional knowledge about $\phi$, we cannot further simplify this problem, we can solve it in linear time. We assume that we intend to solve this system by an error of $\epsilon_\lambda$ and $\epsilon_\mathbf{T}$ for $\lambda$ and $\mathbf{T}$, respectively. Then, we can find a solution through binary search in $O(n\log\frac{1}{\epsilon_\lambda}\log\frac{1}{\epsilon_\tau})$ time. However, many special cases of $\phi$ exist where we can find closed-form solutions for the system and compute $\mathbf{T}$ in $O(n)$ time. We have included the general solution and the special case analyses in the supplementary material. 

\subsubsection{Hardware-Dependent Compression}
At this stage, TOCO compresses the model according to the tolerances computed in the previous section and the characteristics of the target hardware. First, parameters are updated in order to minimize $\Phi$ without violating the tolerance constraints derived in the previous section. TOCO also constrains optimization of parameters to the set $\mathcal{C}$ and encodes them by assigning them symbols, as described in section \ref{sec:comp_vs_hw}. Then, the encodings are tuned by modifying the assignments of some parameters to ensure efficient deployment on the target hardware.

\textbf{Optimization: } Here, parameters are first updated to minimize $\Phi$ under the tolerance constraints. Since $\Phi$ is decomposable, this optimization is simplified to:
\begin{gather}
    \min_{\omega_i\in\mathcal{C}} \phi(\omega_i)\\
    |\omega_i-\omega_i^0|\leq\tau_i
\end{gather}
Since $\phi(\omega_i)$ are optimized independently for all $i$, this optimization can be solved quickly. TOCO sorts $\mathcal{C}$ and starts from the nearest element to $\omega_i^0$, testing elements of $\mathcal{C}$ in the descent direction of $\phi$. When $\phi$ is no longer reduced or elements are further than $\tau_i$ from $\omega_i^0$, the algorithm stops. The value of the parameter is set to the solution of the optimization and its corresponding symbol is assigned to it.

\textbf{Tuning the encodings: }TOCO needs to encode parameters without violating the tolerances such that the network can be efficiently deployed to the hardware. Efficient encoding, as discussed previously, is achieved through enforcing a set of rules informed by the limitations of the hardware. These rules are usually defined as a set of groupings of parameters often based on proximity in the layer matrices \cite{yu2017scalpel,li2016pruning}. The size of groups has to be compatible with the memory access sizes or SIMD width of the hardware. These are, respectively, the number of parameters that can be read from the memory and computed by the target hardware. For example, group size on GPUs is the largest possible to utilize their massive parallelism, that is, equal to the size of a filter in convolution layers. Conversely on a microcontroller with two computation cores, it might be efficient to choose the group size of $2$ to improve the parallelism for this hardware \cite{yu2017scalpel}.

To account for the hardware then, parameters in a group are encoded similarly. This often means that they are encoded with symbols of the same size or have a similar precision. If the previous step encoded some parameters in a group with small symbols or high-precision, then all parameters in that group are updated with similar encodings, i.e. symbol size or precision. These updates find the nearest element in $\mathcal{C}$ that satisfy this requirement and the tolerance constraint for each parameter. We demonstrate this stage with examples of compression using pruning and layer-wise quantization.

\textbf{Pruning} removes unnecessary parameters. TOCO does this by setting parameters like $\omega_i$ that satisfy the pruning condition, that is $|\omega_i|\leq\tau_i$, to zero. For the sake of hardware efficiency, TOCO needs to decide whether to prune for a group of parameters instead of one-by-one. Here, if all parameters in a group satisfy the pruning condition the group is pruned, as illustrated in Figure \ref{fig:pruning_groups}.

\begin{figure}[ht]
    \centering
    \includegraphics[width=0.6\linewidth]{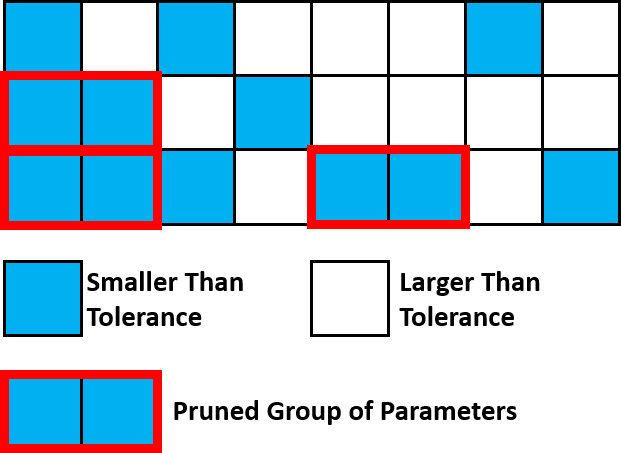}
    \caption{Grouping of parameters in a layer matrix for pruning with a group size of $2$}
    \label{fig:pruning_groups}
\end{figure}

\textbf{Layer-wise Quantization} encodes parameters of each layer with symbols with the same number of bits. TOCO thus defines each group as the parameters in a layer. Parameters in each of these groups are assigned symbols of the same length. This length is the longest symbols size assigned to the parameters of the group before tuning.

After these steps, elements of $\mathcal{C}$ may be tuned to enhance accuracy. This operation is case-specific to the compression method, independent from the steps performed by TOCO. 

The three steps described in this section reduce the size of a DNN model such that the resulting model can be efficient on a target hardware. If the resulting model size is not sufficient for the hardware, TOCO may repeat these steps to further reduce the model size. In the next section, we discuss application of TOCO in more detail and discuss how its hyper-parameters are computed.

\section{Application}
\label{sec:application}



In this section, we will discuss the flow of TOCO and its practical aspects in application to compress a DNN model. As mentioned at the end of the previous section, when compressing a model, the three steps of TOCO are iteratively applied to it. In each iteration, TOCO uses a loss bound ($\overline{\ell}$) to limit the accuracy lost during compression. It also uses a loss model, that is an estimation of the loss function, which needs to be accurate in the neighborhood where we solve the optimization of equation \ref{eq:main_optimization}. In the rest of this section, we will present how TOCO selects $\overline{\ell}$ and ensures the accuracy of the estimation of the loss function. 

\textbf{Loss bound: }TOCO receives an uncompressed model like $\mathbf{W}^0$ as input to compress and uses its loss as the initial value for $\overline{\ell}$. In each subsequent iteration, TOCO evaluates the network loss. If the loss is equal to $\overline{\ell}$ up to an error of $\epsilon$, it means the model cannot be compressed further under this constraint. Therefore, it multiplies $\overline{\ell}$ by a factor of $\sigma>1$. This way, TOCO gradually approximates the smallest loss bound that provides sufficient compression.

\textbf{Modeling the loss function: }TOCO uses the Taylor expansion to locally model the loss function. We ensure the accuracy of this model by introducing an upper bound on the tolerances like $\delta$. We detail application of this bound in the supplementary material. Similar to $\overline{\ell}$, TOCO finds $\delta$ by starting from an initial guess and updating it each iteration. If in an iteration, the loss after encoding the parameters is larger than $\overline{\ell}$, TOCO assumes that the loss model was inaccurate and reduces $\delta$. Otherwise it may be increased. Increasing $\delta$ allows faster compression, but we make sure that $\delta$ is never too large (using $\overline{\delta}$) so that the algorithm can roll back quickly after inaccurately modeling loss. 

Here, we construct the flow of TOCO. Details of this flow have been depicted in algorithm \ref{alg:application}. As this algorithm shows, TOCO iteratively computes tolerances and uses the results to encode the parameters into $\mathbf{W}^*$. After encoding, it evaluates the loss. If loss is within the bound $\overline{\ell}$, the new network parameters are recorded and otherwise they are discarded. The algorithm further updates $\delta$ and $\overline{\ell}$ according to the loss. If $\delta$ becomes too small or $\overline{\ell}$ too large the algorithm stops.

\begin{algorithm}[ht]
   \caption{Application of the Framework}
   \label{alg:application}
   \small
   \renewcommand{\baselinestretch}{0.97} 
\begin{algorithmic}[1]
  \STATE Initialize\quad$\overline{\delta}>\delta$,\quad$\overline{\sigma}>\sigma>1$,\quad$k = 0$
   \STATE $\overline{\ell} = \mathcal{L}(\mathbf{W}_0)$
   \WHILE{$\overline{\ell}\leq \overline{\sigma}\mathcal{L}(\mathbf{W}_0)$\quad or\quad$\delta > \epsilon$\quad or\quad$k<K$}
   \STATE $\mathbf{T} = calculate\textunderscore tolerances(\mathbf{W}_{k-1}, \overline{\ell}, \delta)$
   \STATE $\mathbf{W}^*=compress(\mathbf{W}_{k-1}, \mathbf{T}, \mathcal{C})$
   \IF{$\mathcal{L}(\mathbf{W}^*) > \overline{\ell}$}
   \STATE $\mathbf{W}_k=\mathbf{W}_{k-1}$
   \STATE $\delta = 0.5\delta$
   \ELSE
   \STATE $\mathbf{W}_k=\mathbf{W}^*$
   \STATE $\delta = min(2\delta, \overline{\delta})$
   \ENDIF
   \IF{$|\overline{\ell}-\mathcal{L}(\mathbf{W}^*)| < \epsilon$}
   \STATE $\overline{\ell} = \sigma\overline{\ell}$
   \ENDIF
   \STATE $k = k + 1$
   \ENDWHILE
\end{algorithmic}
\end{algorithm}

\section{Experiments}
We demonstrate the general applicability of our framework using a wide range of datasets and network architectures.

\subsection{Datasets}
We use a set of prominent classification dataset to evaluate the proposed approach. Table \ref{tab:datasets} lists these datasets. \blue{Random was generated randomly using the method proposed by Guyon \cite{guyon2003design}. The dataset comprises 100K, 1000-dimensional vectors. It is trained on a single layer network with sigmoid non-linearity to solve a binary classification problem, using Adam with l2 regularization to $95\%$ accuracy.} ImageNet \cite{deng2009imagenet} comprises $224\times224$ photos of everyday items belonging to $1000$ classes. This dataset contains $50000$ images for validation. We will use $10$-fold cross validation for compression and testing. We use the VGG16 \cite{simonyan2014very} network to test this dataset. The trained model has been downloaded from the existing keras trained models \cite{keras}. MNIST \cite{deng2012mnist} is a set of $60000$ hand-written $28\times 28$ images of digits from $10$ classes, which we have trained on LeNet-5 \cite{lecun1998gradient}. CIFAR-10 contains images of objects from real world, belonging to $10$ classes. We have trained both ResNet50 \cite{he2016deep} and a simplified version of VGG16 \cite{hubara2016binarized} for this dataset. We also use this simplified VGG16 for the SVHN dataset \cite{netzer2011reading}. This dataset contains $600000$ images of digits from the real world, with $10$ classes. 

\begin{table}[tb]
\caption{The datasets used in the experiments}
\label{tab:datasets}
\vskip 0.15in
\begin{center}
\begin{small}
\begin{sc}
\begin{tabular}{lccc}
\toprule
Dataset & Network & Accuracy \\
\midrule
Random      & Single-layer & $95\%$ \\
ImageNet    & VGG16&   $92\%$ (top5) \\
MNIST       & LeNet-5& $99.3\%$\\
CIFAR-10    & ResNet50& $92\%$ \\
CIFAR-10    & VGG & $90\%$\\
SVHN     & VGG& $97.5\%$\\
\bottomrule
\end{tabular}
\end{sc}
\end{small}
\end{center}
\vskip -0.1in
\end{table}

\subsection{Results and Discussion}
In this section, we present several experiments to demonstrate the applicability of the proposed framework and its advantages. \blue{First, we verify the effectiveness of the sample selection in TOCO using the Random dataset.} Then, we show the limitations of heuristics in identifying sensitive parameters in compression. We then compare TOCO with a pruning method based on heuristics and show that it can outperform it. We also compare TOCO against aggressive quantization that enforce encoding policies in large granularities. We show that allowing fine-grained definition of encoding policies can result in higher compression rates. We will then evaluate TOCO on adapting to different hardware architectures and show that it can outperform existing methods as well. As TOCO uses gradients for its tolerance analysis, we will compare it against a gradient-based pruning method and show that it performs similarly at lower cost. Finally, we compare TOCO with another automated pruning method based on meta-learning and show that TOCO performs similarly while it is more general and computationally cost effective.

\textbf{Effects of sample selection: }
\blue{This compression methodology based on active learning was tested using the Random dataset. QBC was used to select the top $1\%$ of the dataset. Next batches of 32 randomly selected samples were used to compute tolerances and compress the model. As baseline, we used the same setup without active-learning (TOCO), and used random batches of size 32 from all of the dataset. The results of this experiment have been depicted in Figure \ref{fig:qbc}. We can see the active-learning-based (active TOCO) method consistently outperforms both the baseline and the value-based method. For the rest of this section we will only test active TOCO method.}
 
 \begin{figure}[tb]
     \centering
     \includegraphics[width=0.6\linewidth]{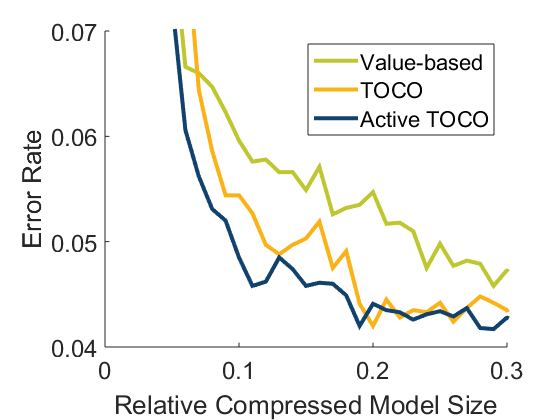}
     \caption{Comparison of TOCO with and without active sample selection and value-based pruning}
     \label{fig:qbc}
 \end{figure}
 
\textbf{Limitations of heuristics: }Many compression methods propose heuristic measures to decide which parameters to perturb when compressing. However, these heuristics can be inaccurate. As an example, here we study the absolute value of parameters as a measure of their importance to maintaining the accuracy which is used by many pruning methods. For this experiment, we compress a trained DNN model using TOCO and fixed-point quantization as the compression method. That is similar to adaptive quantization \cite{khoram2018adaptive}, in each iteration, TOCO reduces the encoding size of parameters until each parameter is quantized with the fewest encoding bits necessary. We further define a measure of importance in this context for parameters:
\begin{align}
    &\mu_{\omega_i} = \frac{K}{\sum_k \rho(\omega_i^k)} \\
    &\rho(\omega_i^k) = 
    \begin{cases}
        1& \phi(\omega_i^k) < \phi(\omega_i^{k-1}) \text{ or } \phi(\omega_i^k)=0 \\
        0& \text{Otherwise}
    \end{cases}
\end{align}
Here, $\mu_{\omega_i}$ measures how fast the iterative process of algorithm \ref{alg:application} minimizes the encoding size of the parameter $\omega_i$. Larger values of $\mu_{\omega_i}$ mean faster elimination of the parameter and correspond to a lower importance. We depict the importance of the parameters in the first fully-connected layer of LeNet-5 versus their initial values at the start of compression, in Figure \ref{fig:importance}. We can see that there is very little correlation between parameter importance and their initial values. Thus, absolute value is not a good measure of the importance of parameters.

\begin{figure}[tb]
\centering
    \includegraphics[width=0.6\linewidth]{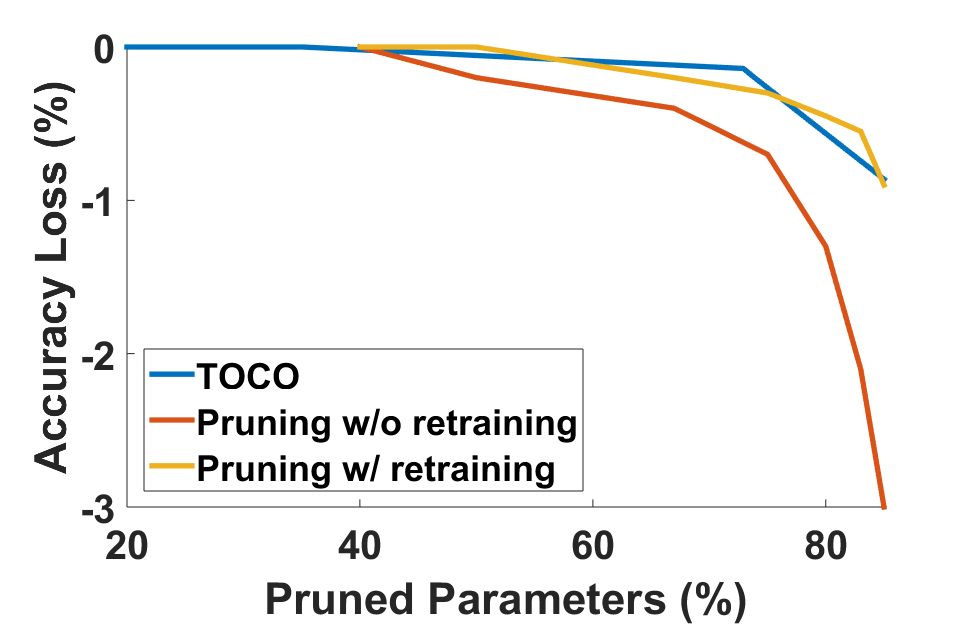}
    \caption{Unconstrained pruning on VGG16 trained on ImageNet}
    \label{fig:pruning}
\end{figure}
\begin{figure*}[ht]
\centering
\subfigure[MNIST]{%
\label{fig:first}%
\includegraphics[width=0.3\linewidth]{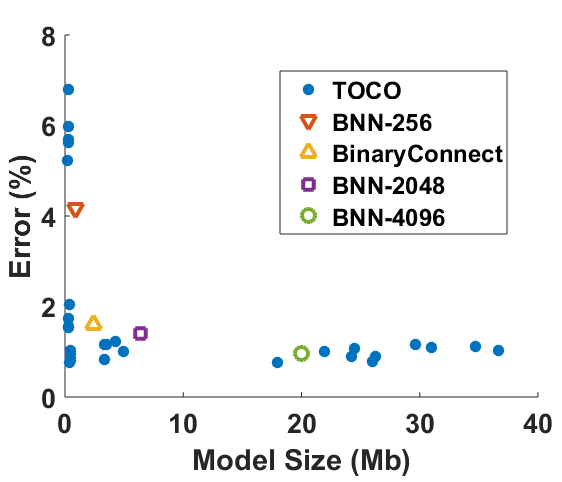}}%
~
\subfigure[CIFAR10]{%
\label{fig:second}%
\includegraphics[width=0.33\linewidth]{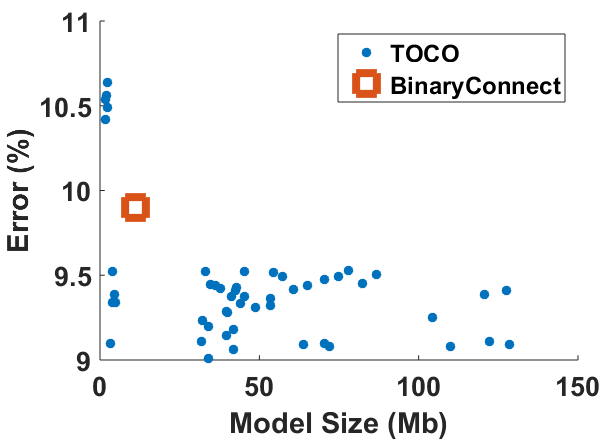}}%
~
\subfigure[SVHN]{%
\label{fig:second}%
\includegraphics[width=0.3\linewidth]{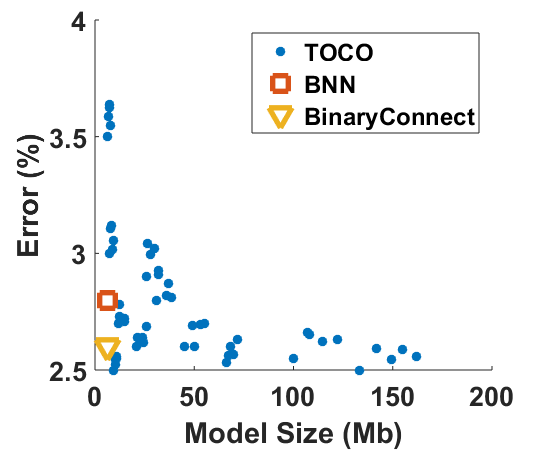}}%

\caption{Comparison of the proposed method in quantization with previous works}

\label{fig:quantization}
\end{figure*}

\begin{figure}[tb]
\centering
    \includegraphics[width=0.6\linewidth]{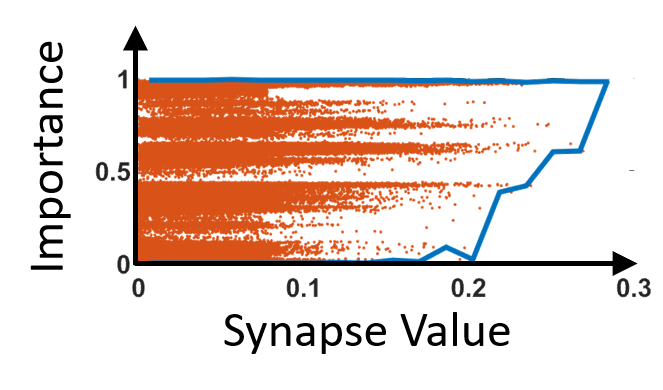}
    \caption{Importance vs. value of parameters}
    \label{fig:importance}
\end{figure}

\textbf{Comparison with heuristic pruning: }We compare the pruning capability of the proposed method with a value-based method \cite{han2015deep} on compressing VGG16 trained on ImageNet. In this experiment, we do not include any hardware constraints and prune parameters solely based on their effect on accuracy, that is group size is set to $1$. We use the tolerances derived in the quantization to determine importance. Following the results of Figure \ref{fig:importance}, we expect the proposed framework to eliminate parameters more effectively and with a smaller loss of accuracy compared to the value-based approach. We examine this by plotting the accuracy loss for the parameter elimination rate in Figure \ref{fig:pruning}.

Figure \ref{fig:pruning} compares the unconstrained pruning performance of the proposed method with the value-based approach in two scenarios, before and after fine tuning. In the first scenario, small parameters are eliminated only and the accuracy is reported. In the second scenario we further compare the results with value-based pruning followed by retraining. We do not retrain for TOCO. As shown, we outperform the first case and perform similar to the second case. We note that the compression is performed in a matter of minutes while retraining on ImageNet may take up to several hours. 

\textbf{Fine-grained quantization: }Next, we apply TOCO to parameter quantization and compare our results with two aggressive quantization methods. Specifically, we compare TOCO with BNN \cite{hubara2016binarized} and BinaryConnect \cite{courbariaux2015binaryconnect} which use one bit to represent parameters. In these experiments, we do not consider any specific hardware, that is, we do not apply the third step of TOCO. This allows us to evaluate the limits of quantization using TOCO. We perform these comparisons for three models: MNIST trained on LeNet-5, and CIFAR10 and SVHN trained on the simplified VGG16 model. We depict the results of these experiments in Figure \ref{fig:quantization}.

In Figure \ref{fig:quantization}, we have depicted the trade-off between size and accuracy of quantized models in multiple passes of the compression algorithm. Here, we start from the models in Table \ref{tab:datasets} and apply the quantization method. Then, we retrain the quantized model in floating-point domain and apply the quantization to the retrained model. We repeat this process three times and depict the different data points we find during this process in the Figure. As we can see all generated data points exist in a narrow band near the pareto-optimal front of the compression. This further emphasizes the efficiency of the proposed method specially when retraining is costly. We can also see that this model presents a lower bound for aggressive compression techniques. As such it may be used as an approximate a baseline to evaluate the efficiency of other methods.

\textbf{Evaluation over different hardware: }A key characteristic of TOCO is how it can adapt to different target hardwares. To demonstrate this, we use it to prune LeNet trained on MNIST for implementation on GPU (GTX Titan X), CPU (Intel Core i7-6700), and microcontroller (ARM Cortex-M4) and compare with a value-based method tailored to hardware implementation called Scalpel \cite{yu2017scalpel}. We follow Scalpel when grouping parameters for each target hardware. For the GPU, Scalpel groups parameters such that all layer matrices are dense after pruning. For the CPU, the same is done for the convolution layers, and a group size of $8$ is used for the fully connected layers to utilize the $8$ cores of the CPU. In the case of the Microcontroller, we use a group size of $2$ for all layers. We have compared the results of the proposed compression method with Scalpel, in Figure \ref{fig:scalpel}.

\begin{figure}[tb]
\centering
\subfigure[Compression rate]{%
\label{fig:first}%
\includegraphics[width=0.4\linewidth]{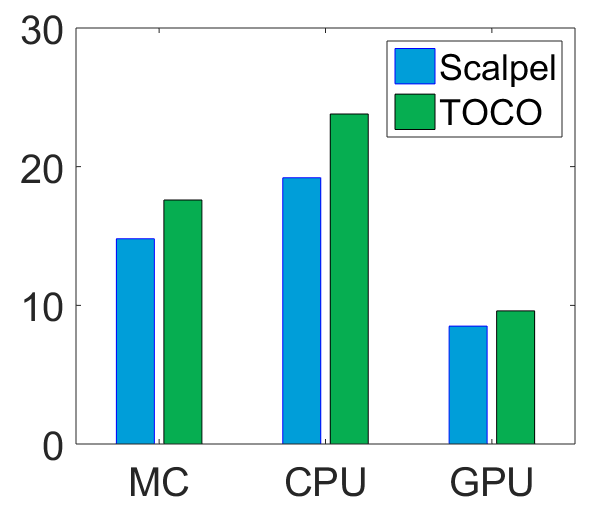}}%
~
\subfigure[Speedup]{%
\label{fig:second}%
\includegraphics[width=0.37\linewidth]{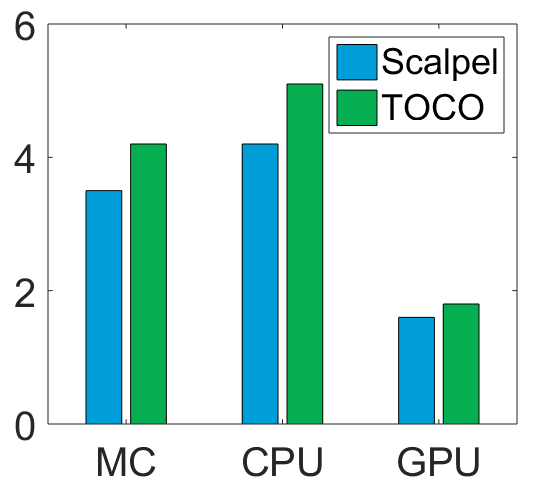}}%
\caption{Comparison of the proposed method with Scalpel on hardware implementation normalized to uncompressed model.}
\label{fig:scalpel}
\end{figure}
Figure \ref{fig:scalpel} depicts the compression rate and speedup of TOCO with Scalpel for the same classification accuracy. We achieve a higher compression rate compared to scalpel taking into account all storage overheads of sparse models. Using these compressed models, we are able to achieve higher speedups on each of the hardware platforms. The speedups have been interpolated using profiles of sparse matrix multiplication on different hardware produced  in previous work by \citet{yu2017scalpel}. 

We further visualize the first fully connected layer after pruning using the two pruning methods for the case of microcontroller in Figure \ref{fig:visual}. It is evident from this Figure that TOCO achieves a higher pruning while maintaining the parallelism required by the hardware. This result is achieved since TOCO can better account for the importance of parameters. We confirm this by visualizing the importance values and the initial parameter values of that layer in Figure \ref{fig:importance_scalpel}.
\begin{figure}[tb]
\centering
\subfigure[TOCO]{%
\label{fig:first}%
\includegraphics[width=0.4\linewidth]{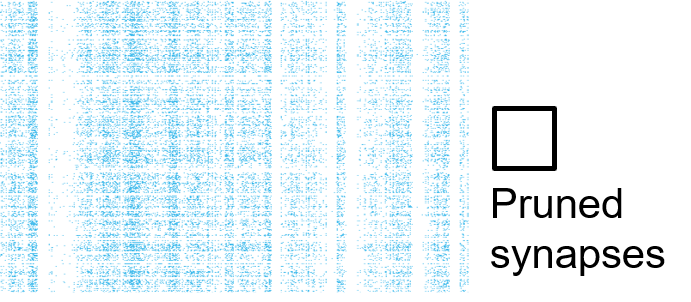}}%
~
\subfigure[Scalpel]{%
\label{fig:second}%
\includegraphics[width=0.4\linewidth]{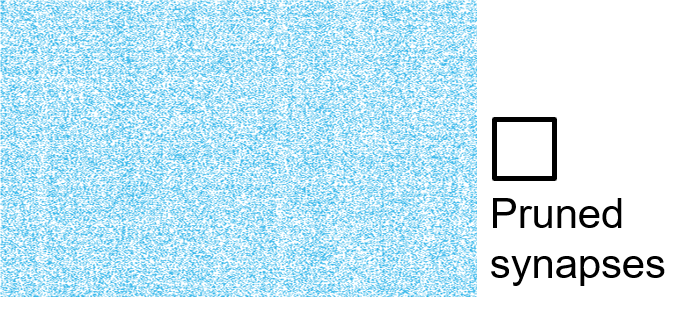}}%
\caption{Visualization of a pruned layer using the proposed method and a value-based approach}
\label{fig:visual}
\end{figure}

\begin{figure}[tb]
\centering
\subfigure[Parameter importance]{%
\label{fig:first}%
\includegraphics[width=0.4\linewidth]{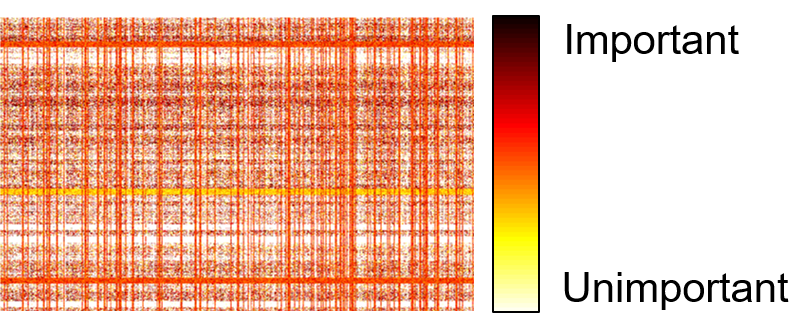}}%
~
\subfigure[Initial values]{%
\label{fig:second}%
\includegraphics[width=0.4\linewidth]{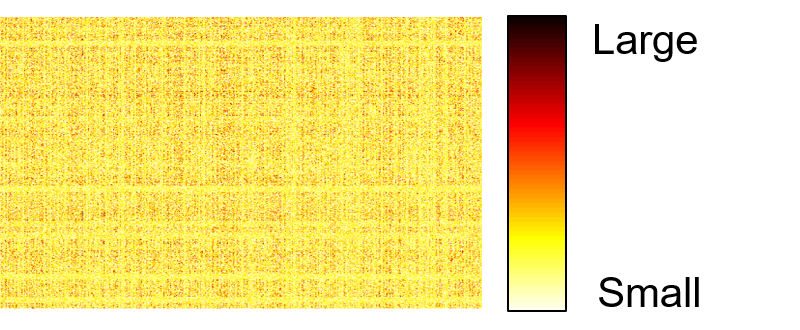}}%
\caption{Importance and parameter values of the first FC layer in LeNet-5}
\label{fig:importance_scalpel}
\end{figure}

\textbf{Comparison with previous gradient methods: }As an example of existing gradient methods we use Optimal Brain Damage (OBD) \cite{lecun1990optimal} which originally introduced the idea of pruning based on the Hessian and is still being applied in practice \cite{theis2018faster}. This method uses the second order gradient to compute a measure of saliency for parameters and prunes low-saliency ones. While this method can be effective, it has the higher cost of the second order gradient. Due to the high computation cost, we had perform this comparison on a smaller model. We have used a two-layer MLP with $16000$ parameters. Figure \ref{fig:obd} shows the error rate of this network for different pruning rates using TOCO and OBD. As this Figure shows, TOCO and OBD perform generally similarly, with TOCO outperforming OBD when fewer parameters remain.

\begin{figure}[tb]
\centering
    \includegraphics[width=0.4\linewidth]{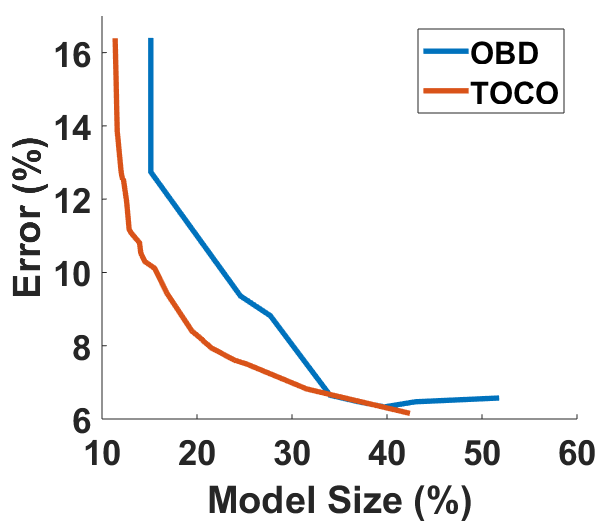}
    \caption{Pruning using TOCO and OBD}
    \label{fig:obd}
\end{figure}

\textbf{Comparison with automated frameworks: }Finally, we compare the proposed approach with an existing automated method called AutoML for Model Compression (AMC) \cite{he2018amc}. Unlike TOCO, AMC does not account for the importance of individual parameters. But, through Reinforcement Learning (RL), it learns the best compression rate for each layer. The limitation of this approach is that it cannot easily adapt other compression problems which cannot be modeled as fully observable Markov Decision Processes (MDPs) limiting the applicability of RL e.g. greatient-based methods like adaptive quantization. The limits of the applying RL with replay buffers as done in AMC for partially observable MDPs has been demonstrated in previous works \cite{hausknecht2015deep}.


For this experiment we prune the ResNet network trained on CIFAR10 dataset and use the same groupings here as we did for the GPU in the last experiment. As such, both methods produce dense layer matrices that can be efficiently deployed to GPUs. We plot the results of this experiment in Figure \ref{fig:amc}. This Figure shows that the TOCO achieves a similar pruning rate as AMC with a small accuracy loss. We note that AMC requires $1$ hour processing time on a TITAN GPU while TOCO performs the pruning in about a quarter of the time using the closed-form solution presented in the supplementary material. Furthermore, TOCO applies to a wide range of compression methods while AMC performs only pruning.

\begin{figure}[tb]
\centering
    \includegraphics[width=0.5\linewidth]{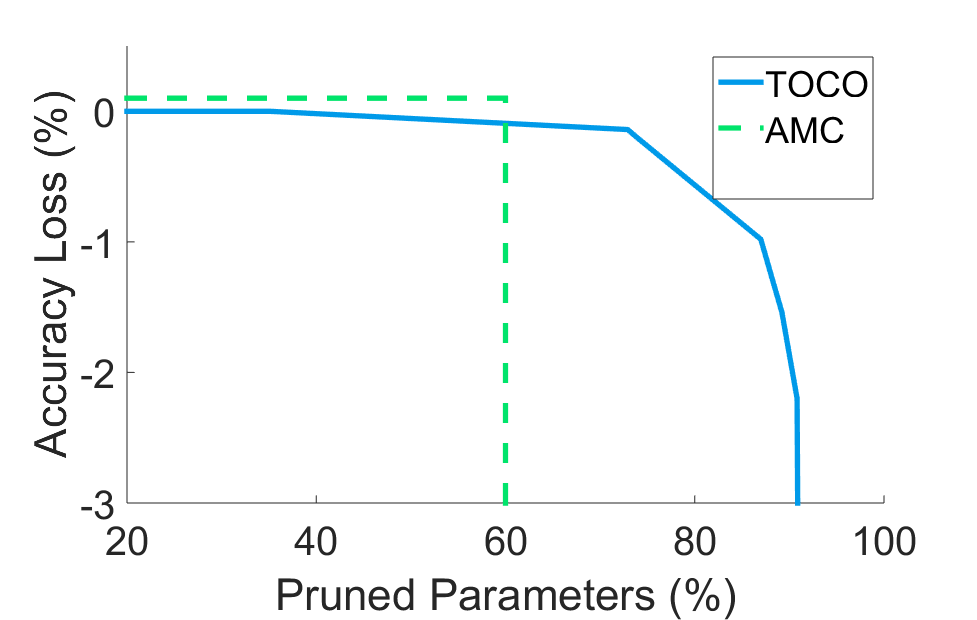}
    \caption{Comparison of the proposed framework with AMC}
    \label{fig:amc}
\end{figure}

\section{Conclusion}
In this work, we presented TOCO, a framework for compressing neural network models for deployment onto various edge computing hardware. In TOCO, compression method is isolated from the hardware through an intermediate tolerance analysis step. As such, it is easily portable across a wide range of target hardware. Furthermore, through its tolerance analysis TOCO learns the individual importances of the parameters. By encoding parameters based on their performance, TOCO is able to maintain accuracy during compression. We showed that TOCO has a wide range of applicability to different compression methods and hardware using comprehensive experiments.


\bibliography{ms}
\bibliographystyle{sysml2019}

\appendix

\section{Supplementary material}
As described in section 3, TOCO solves an optimization problem to compute the tolerances for network parameters. Here, we first present a general solution for this problem based on a binary search. Then, we discuss two cases where we can calculate a closed-form solution. 

\subsection{General Solution}
TOCO optimizes equation 2 by solving the following equation system for $\mathbf{T}$ and $\lambda$. In solving this system, we assume that we can tolerate an error of $\epsilon_\tau$ for the values of elements of $\mathbf{T}$ and an error of$\epsilon_\lambda$ for the value $\lambda$.
\begin{gather}
    \label{eq:tau}
    \nabla_{\mathbf{T}}\phi+\lambda G = 0\\
    \label{eq:lambda}
    \lambda(\mathbf{G}^TT-\Delta\overline{\ell}) = 0\\
    \lambda \geq 0
\end{gather}
TOCO solves this system by guessing a value for $\lambda$, solving equation \ref{eq:tau}, and then refining the guess based on equation \ref{eq:lambda}. We will first discuss the process of guessing and updating $\lambda$ based on the solution of equation \ref{eq:tau}, then we present the details of solving this equation. 

\textbf{Computing $\lambda$: }The process of guessing $\lambda$ and refining it follows a binary search process. As described in algorithm \ref{alg:subproblem}. For this process, we assume that $0\leq\lambda\leq\overline{\lambda}$. In lemma \ref{lemma:lambda}, we prove that $\overline{\lambda}$ exists and compute its value.
\begin{algorithm}[ht]
   \caption{Computing Tolerances}
   \label{alg:subproblem}
\begin{algorithmic}[1]
    \STATE {\bfseries Input:} $\mathbf{G}$, $\overline{\lambda}$, $\epsilon_\lambda$
   \STATE $k=1$
   \STATE $\lambda_0=0$
   \STATE $\lambda_{min}=0$
   \STATE $\lambda_{max}=\overline{\lambda}$
   \REPEAT
   \STATE $\lambda_k=\frac{\lambda_{min}+\lambda_{max}}{2}$
   \STATE Solve $\nabla_\mathbf{T}\phi=-\lambda_k\mathbf{G}$ for $\mathbf{T}$
   \IF{$\mathbf{G}^T\mathbf{T}>\Delta\overline{\ell}$}
   \STATE $\lambda_{max} = \lambda_k$
   \ELSE
   \STATE $\lambda_{min} = \lambda_k$
   \ENDIF
   \STATE $k = k+1$
   \UNTIL{$|\lambda_k-\lambda_{k-1}|<\epsilon_\lambda$}
\end{algorithmic}
\end{algorithm}
\begin{lemma}
\label{lemma:lambda}
Assuming  $\mathbf{G}_\iota=\max(\mathbf{G})$ then:
\begin{equation}
    \lambda \leq \overline{\lambda} = \frac{1}{\mathbf{G}_\iota}\frac{\partial\phi}{\partial\tau} (\frac{\Delta\overline{\ell}}{\lVert\mathbf{G}\rVert_1})
\end{equation}
\end{lemma}
\begin{proof}
Since all $\phi$ are strictly convex, then $f=(\frac{\partial\phi}{\partial\tau})^{-1}$ is decreasing. Therefore, $\tau_\iota=f(-\lambda\mathbf{G}_\iota)=\min(\mathbf{T})$. Consequently, We can rewrite equation \ref{eq:tau}:
\begin{gather*}
    \Delta\overline{\ell} = \sum_i \mathbf{G}_i\tau_i \geq \sum_i\mathbf{G}_i\tau_\iota = \tau_\iota\lVert\mathbf{G}\rVert_1=f(-\lambda\mathbf{G}_\iota)\lVert\mathbf{G}\rVert_1\\
    \Rightarrow -\lambda\mathbf{G}_\iota \geq \frac{\partial\phi}{\partial\tau} (\frac{\Delta\overline{\ell}}{\lVert\mathbf{G}\rVert_1})\Rightarrow \lambda \leq \frac{1}{\mathbf{G}_\iota}\frac{\partial\phi}{\partial\tau} (\frac{\Delta\overline{\ell}}{\lVert\mathbf{G}\rVert_1})
\end{gather*}
\end{proof}
Algorithm \ref{alg:subproblem} sets an upper and a lower bound for $\lambda$. Then, in each iteration, it assumes the midpoint in this range for $\lambda$ and solves equation \ref{eq:tau} for $\mathbf{T}$. If the value of $\mathbf{G}^T\mathbf{T}-\Delta\overline{\ell}$ for the computed $\mathbf{T}$ is positive, it means that $\lambda$ overestimated. Otherwise, it was underestimated. Consequently, either the upper bound or the lower bound of $\lambda$ is updated. These steps are repeated until the convergence of $\lambda$, determined by $\epsilon_\lambda$. As a result of the binary search, algorithm \ref{alg:subproblem} takes at most $O(\log\frac{1}{\epsilon_\lambda})$ iterations.

\textbf{Computing $\mathbf{T}$: }Solving equation \ref{eq:tau} for $\mathbf{T}$ follows a similar approach as $\lambda$. Specifically, it solves the following equation for all $i$ using a binary search.
\begin{equation}
    \frac{\partial}{\partial} \phi(\tau_i)+\lambda g_i = 0
\end{equation}
In this solution, for all $i$ we assume $0\leq\tau_i\leq\overline{\tau}$ and compute $\tau_i$ using algorithm \ref{alg:tau} to an error of $\epsilon_\tau$. Here the value of $\overline{\tau}$ is $\min(\delta, \overline{\omega})$, where $\delta$ refers to algorithm 1 in the main paper.
\begin{algorithm}[ht]
   \caption{Computing $\tau_i$}
   \label{alg:tau}
\begin{algorithmic}[1]
    \STATE {\bfseries Input:} $\mathbf{G}$, $\overline{\tau}$, $\epsilon_\tau$, $\lambda$
   \STATE $k=1$
   \STATE $\tau^0=0$
   \STATE $\tau_{min}=0$
   \STATE $\tau_{max}=\overline{\tau}$
   \REPEAT
   \STATE $\tau^k=\frac{\tau_{min}+\tau_{max}}{2}$
   \IF{$\frac{\partial}{\partial} \phi(\tau_i)+\lambda g_i > 0$}
   \STATE $\tau_{max} = \tau^k$
   \ELSE
   \STATE $\tau_{min} = \tau^k$
   \ENDIF
   \STATE $k = k+1$
   \UNTIL{$|\tau^k-\tau^{k-1}|<\epsilon_\tau$}
\end{algorithmic}
\end{algorithm}
This algorithm also uses a binary search process and thus converges in t most $O(\log\frac{1}{\epsilon_\tau})$. Since this is repeated for all $i$, solving equation \ref{eq:tau} is solved in $O(n\log\frac{1}{\epsilon_\tau})$ time. This results in an overall $O(n\log\frac{1}{\epsilon_\tau}\log\frac{1}{\epsilon_\lambda})$ time or solving the equation system of equations \ref{eq:tau} and \ref{eq:lambda}.

Absent any additional information about $\phi$, this solution is used to quickly compute the tolerances. Next, we discuss cases where $\phi$ is known using two special cases where this solution can be forgone in favor of a closed-form solution.

\subsection{Special Cases}
We study two cases here where the function $\phi$ is known. Specifically, we present quantization as described by \citet{khoram2018adaptive} and \citet{arabi2004low}.

\textbf{Special Case 1. }\citet{khoram2018adaptive} uses the following logarithmic function to construct its measure of complexity for the purpose of fixed-point quantization of parameters:
\begin{equation}
    \phi(\tau_i) = -\log\tau_i
\end{equation}
In this case, it is easy to see the the equation system has a unique solution:
\begin{equation}
    \tau_i=\frac{\Delta\overline{\ell}}{ng_i}
\end{equation}
\textbf{Special Case 2. }Conversely, \citet{arabi2004low} quantizes network parameters using floating-point values and uses a quadratic function to build a measure of complexity.
\begin{equation}
    \phi(\tau_i)=h_{ii}(\omega_i-\mathcal{C}_j^i+\alpha_i\tau_i)^2
\end{equation}
Here, $h_{ii}$ are elements of the hessian of the loss function used to train the network, $\mathcal{C}_j^i$ are elements of the set $\mathcal{C}$, and $\alpha_i$ identify the descent direction of $\phi$. In this case, the following computes the tolerance values.
\begin{gather}
    \tau_i = \frac{g_i}{2h_{ii}}\lambda+\frac{C_j^i-\omega_i}{\alpha_i}\\
    \lambda = \frac{\Delta\overline{\ell}+\sum_i\frac{\omega_i-C_j^i}{\alpha_ig_i}}{\sum_i\frac{g_i^2}{2h_{ii}}}
\end{gather}



\end{document}